\documentclass[letterpaper,11pt]{article}
\usepackage[margin=1in]{geometry}
\usepackage[bookmarks, colorlinks=true, plainpages = false, citecolor = blue,linkcolor=red,urlcolor = blue, filecolor = blue]{hyperref}
\usepackage{url}
\usepackage{amsmath,amsfonts,amsthm,amssymb,bm, verbatim,dsfont,mathtools}
\usepackage{algorithm,algorithmic,color,graphicx,appendix}
\usepackage{subfigure}
\usepackage{etoolbox}
\usepackage{tikz}
\usepackage{xr,xspace}
\usepackage{todonotes}
\usepackage{paralist}
\usepackage{caption}
\theoremstyle{plain}
\newtheorem{theorem}{Theorem}
\newtheorem{lemma}{Lemma}
\newtheorem{claim}{Claim}
\newtheorem{proposition}{Proposition}
\newtheorem{corollary}{Corollary}
\theoremstyle{definition}
\newtheorem{definition}{Definition}

\newtheorem*{remark*}{Remark}

\newcommand\1[1]{\mathbb{I}_{\left\{#1\right\}}}

\usepackage{xspace,prettyref}

\newcommand{\ones}{\mathbf 1}

\newcommand{\reals}{{\mathbb{R}}}

\newcommand{\Expect}{\mathbb{E}}

\newcommand{\eg}{e.g.\xspace}

\newcommand{\Tr}{{\rm Tr}}
\newrefformat{eq}{(\ref{#1})}
\newrefformat{chap}{Chapter~\ref{#1}}
\newrefformat{sec}{Section~\ref{#1}}
\newrefformat{algo}{Algorithm~\ref{#1}}
\newrefformat{fig}{Fig.~\ref{#1}}
\newrefformat{tab}{Table~\ref{#1}}
\newrefformat{rmk}{Remark~\ref{#1}}
\newrefformat{clm}{Claim~\ref{#1}}
\newrefformat{def}{Definition~\ref{#1}}
\newrefformat{cor}{Corollary~\ref{#1}}
\newrefformat{lmm}{Lemma~\ref{#1}}
\newrefformat{prop}{Proposition~\ref{#1}}
\newrefformat{app}{Appendix~\ref{#1}}
\newrefformat{hyp}{Hypothesis~\ref{#1}}
\newrefformat{thm}{Theorem~\ref{#1}}

\newcommand{\calL}{{\mathcal{L}}}

\newcommand{\calS}{{\mathcal{S}}}

\newcommand{\calU}{{\mathcal{U}}}

\newcommand{\ML}{{\sf ML}\xspace}
\newcommand{\PL}{{\sf PL}\xspace}
\newcommand{\MM}{{\sf MM}\xspace}
\newcommand{\EM}{{\sf EM}\xspace}
\newcommand{\IB}{{\sf IB}\xspace}
\newcommand{\FB}{{\sf FB}\xspace}

\newcommand{\IR}{\mathbb{R}}

\title{Minimax-optimal Inference from Partial Rankings}
\author{
Bruce Hajek \\
UIUC \\
{b-hajek@illinois.edu}
\and
Sewoong Oh\\
UIUC\\
{swoh@illinois.edu}
\and
Jiaming Xu  \\
UIUC\\
{jxu18@illinois.edu}}

\begin{document}

\maketitle

\begin{abstract}
This paper studies the problem of inferring a global preference based on the partial rankings provided by many users over different subsets of items according to the Plackett-Luce model.  A question of particular interest is how to optimally assign items to users for ranking and how many item assignments are needed to achieve a target estimation error.  For a given assignment of items to users, we first derive an oracle lower bound of the estimation error that holds even for the more general Thurstone models. Then we show that the Cram\'er-Rao lower bound and our upper bounds inversely depend on the spectral gap of the Laplacian of an appropriately defined comparison graph. When the system is allowed to choose the item assignment, we propose a random assignment scheme. Our oracle lower bound and upper bounds imply that it is minimax-optimal up to a logarithmic factor among all assignment schemes and the lower bound can be achieved by the maximum likelihood estimator as well as popular rank-breaking schemes that decompose partial rankings into pairwise comparisons. The numerical experiments corroborate our theoretical findings.

\end{abstract}

\section{Introduction}
%

Given a set of individual preferences from multiple decision makers or judges,
we address the problem of computing a consensus ranking that best represents the preference of the population collectively.
This problem, known as rank aggregation, has received much
attention across various disciplines including statistics, psychology, sociology, and computer science, and has found numerous applications
including elections, sports, information retrieval, transportation, and marketing \cite{TRANSPORTATION,MARKETING,McF80transportation,BIOLOGY}.
While consistency of various rank aggregation algorithms has been studied
when a growing number of sampled partial preferences is observed 
 over a fixed number of items \cite{SY99,DMJ10}, little is known in the high-dimensional setting
where the number of items and number of observed partial rankings scale simultaneously, which
arises in many modern datasets.
Inference becomes even more challenging when each individual provides limited information. 
 For example, in the well known Netflix challenge dataset,
480,189 users submitted ratings on 17,770 movies,   
but on average a user rated only $209$ movies. To pursue a rigorous study in the high-dimensional setting, we assume that
users provide partial rankings over subsets of items generated according to the popular Plackett-Luce (\PL) model \cite{hunter04} from some hidden preference vector over all the items and are interested in estimating the preference vector 
(see \prettyref{def:PLModel}).

Intuitively, inference becomes harder when few users are available, or each user is assigned few items to rank, meaning fewer observations.
The first goal of this paper is to quantify the number of item assignments needed to achieve a target estimation error. 
Secondly, in many practical scenarios such as crowdsourcing, the systems have the control over  the item assignment.
For such systems, a natural question of interest is how to optimally assign the items for a given budget on the total number of item assignments. 
Thirdly, a common approach in practice to deal with partial rankings is
to break them into pairwise comparisons and apply the state-of-the-art rank aggregation methods specialized for pairwise comparisons \cite{Souriani13,AzariSoufiani_icml14}.
It is of both theoretical and practical interest to 
understand how much the performance degrades when rank breaking schemes are used.

\paragraph*{Notation.}
For any set $S$, let $|S|$ denote its cardinality. Let $s_1^n=\{s_1, \ldots, s_n \}$ denote a set with $n$ elements.
For any positive integer $N$, let $[N]=\{1, \ldots, N\}$. We use standard big $O$ notations, e.g., for any sequences $\{a_n\}$ and $\{b_n\}$, $a_n=\Theta(b_n)$ if there is an absolute constant $C>0$ such that $1/C \le a_n/ b_n \le C$. For a partial ranking $\sigma$ over $S$, i.e., $\sigma$ is a mapping from $[|S|]$ to $S$, let $\sigma^{-1}$ denote the inverse mapping.
All logarithms are natural unless the base is explicitly specified. We say a sequence of events $\{A_n\}$ holds with high probability if $\mathbb{P}[A_n] \ge 1-c_1 n^{-c_2}$ for two positive constants $c_1,c_2$.

\subsection{Problem setup}
We describe our model in the context of recommender systems, but it is applicable to other systems with partial rankings.
Consider a recommender system with $m$ users indexed by $[m]$ and $n$ items indexed by $[n]$.
For each item $i \in [n]$, there is a hidden parameter $\theta^\ast_i$ measuring the underlying preference.
Each user $j$, independent of everyone else, randomly generates a partial ranking $\sigma_j$ over a subset of items $S_j \subseteq [n]$  according to the \PL model with the underlying preference vector $\theta^\ast=(\theta^\ast_1,\ldots, \theta^\ast_n)$.
 \begin{definition}[\PL model] \label{def:PLModel}
 A partial ranking $\sigma: [|S|] \to S$ is generated from $\{\theta^\ast_i, i \in S\}$ under the \PL model
 in two steps: (1) independently assign each item $i \in S$ an unobserved value $X_i,$ exponentially distributed with mean $e^{-\theta^\ast_i}; $
 (2) select $\sigma$ so that $X_{\sigma(1)} \le X_{\sigma(2)} \le \cdots \le X_{\sigma(|S|)}$.
\end{definition}

The \PL model can be equivalently described in the following sequential manner. 
To generate a partial ranking $\sigma$, first select $\sigma(1)$ in $S$ randomly from  the distribution 
${e^{\theta^\ast_i} }/\big( \sum_{i'\in S} e^{\theta^\ast_{i'}} \big) $; secondly, select  $\sigma(2)$ in $S\setminus\{\sigma(1)\}$
with the probability distribution ${e^{\theta^\ast_i} }/\big({\sum_{{i'} \in S\setminus\{\sigma(1)\}  } e^{\theta^\ast_{i'}} }\big)$; 
continue the process in the same fashion until all the items in $S$ are assigned. The \PL model is a special case of the following class of models.
 \begin{definition}[Thurstone model, or random utility model (RUM) ] \label{def:ThurstoneModel}
A partial ranking $\sigma: [|S|] \to S$ is generated from $\{\theta^\ast_i, i \in S\}$
under the Thurstone model for a given CDF $F$  in  two steps: (1) independently assign each item $i \in S$ an unobserved utility $U_i,$
with CDF $F(c-\theta^*_i); $
 (2) select $\sigma$ so that $U_{\sigma(1)} \ge U_{\sigma(2)} \ge \cdots \ge U_{\sigma(|S|)}$.
 \end{definition}
To recover the \PL model from the Thurstone model, take $F$ to be the CDF for the standard Gumbel distribution:
$F(c) = e^{-(e^{-c})}$.   Equivalently, take $F$ to be the CDF of $-\log(X)$ such that $X$ has the exponential
distribution with mean one.   For this choice of $F,$ the utility $U_i$ having CDF $F(c-\theta^*_i),$ is equivalent to $U_i = -\log(X_i)$
such that $X_i$ is exponentially distributed with mean $e^{-\theta^*_i}.$
The corresponding partial permutation $\sigma$ is such that $X_{\sigma(1)} \le X_{\sigma(2)} \le \cdots \le X_{\sigma(|S|)},$
or equivalently, $U_{\sigma(1)} \ge U_{\sigma(2)} \ge \cdots \ge U_{\sigma(|S|)}.$  (Note the opposite ordering of $X$'s and $U$'s.)

Given the observation of all partial rankings $\{\sigma_j\}_{j \in [m]}$ over the subsets $\{ S_j\}_{j \in [m]}$ of items, the task is to infer the underlying preference vector $\theta^\ast$.
For the \PL model, and more generally for the Thurstone model, we see that $\theta^\ast$ and $\theta^\ast+a\ones$ for any $a \in \reals$ are statistically indistinguishable,
where $\ones$ is an all-ones vector. Indeed, under our model, the preference vector $\theta^\ast$ is the equivalence class 
 $[\theta^\ast]=\{ \theta: \exists a \in \reals,  \theta=\theta^\ast + a\ones \}$. To get a unique representation of the equivalence class, we assume $\sum_{i=1}^n \theta^\ast_i =0$.
Then the space of all possible preference vectors is given by
$
\Theta=\{\theta \in  \reals^n: \sum_{i=1}^n \theta_i=0 \}.
$ Moreover, if $\theta^*_i - \theta^\ast_{i'}$ becomes arbitrarily large for all $i' \neq i$, then with
high probability item $i$ is ranked higher than any other item $i'$ and there is no way to estimate $\theta_i$
to any accuracy. Therefore, we further put the constraint that $\theta^\ast \in [-b, b]^n$ for some $b \in \reals$ and define $\Theta_b=\Theta \cap [-b,b]^n$.
The parameter $b$ characterizes the dynamic range of the underlying preference. In this paper, we assume $b$ is a fixed constant. As observed in \cite{Shah12}, if $b$ were scaled with $n$, then it would be easy to rank items with high preference versus items with low preference and one can focus on ranking items with close preference.

We denote the number of items assigned to user $j$ by $k_j:=|S_j|$ and the average number of  assigned items per use by $k=\frac{1}{m} \sum_{j=1}^m k_j;$ parameter $k$ may scale with $n$ in this paper. We consider two scenarios for generating the subsets $\{S_j\}_{j=1}^m$: the random item assignment case where the $S_j$'s are chosen independently and uniformly at random from all possible subsets of $[n]$ with sizes given by the $k_j$'s, and the deterministic item assignment case where the
$S_j$'s are chosen deterministically.


Our main results depend on the structure of a weighted undirected graph $G$ defined as follows.
\begin{definition} [Comparison graph $G$] \label{def:ComparisonGraph}
Each item $i \in [n]$ corresponds to a vertex $i \in [n]$. For any pair of vertices $i, i'$, there is a weighted edge between them if there exists a user who ranks both items $i$ and $i'$; the weight equals $\sum_{j: i, i' \in S_j} \frac{1}{k_j-1}$.
\end{definition}
Let $A$ denote the weighted adjacency matrix of $G$. 
Let $d_i=\sum_{j} A_{ij},$ so $d_i$ is the number of users who rank item $i,$ and without loss of generality assume $d_1 \le d_2 \le \cdots \le d_n$. 
Let $D$ denote the $n \times n$ diagonal matrix formed by $\{d_i, i\in [n]\}$ and define the graph Laplacian $L$ as $L=D-A$.
Observe that $L$ is positive semi-definite and the smallest eigenvalue of $L$ is zero with the corresponding eigenvector given by the normalized all-one vector. Let $0=\lambda_1 \le \lambda_2 \le \cdots \le \lambda_n$ denote the eigenvalues of $L$ in ascending order.

\paragraph*{Summary of main results.}
\prettyref{thm:Oracle} gives a lower bound for the estimation error that scales as $\sum_{i=2}^n \frac{1}{d_i}$.
The lower bound is derived based on a genie-argument and holds for both the \PL model and the more general Thurstone model. 
\prettyref{thm:CramerRao} shows that the Cram\'er-Rao lower bound scales as $\sum_{i=2}^n \frac{1}{\lambda_i}$.
 \prettyref{thm:MLEUpperBound} gives an upper bound for the squared error of the maximum likelihood (\ML) estimator that scales
as $\frac{ mk \log n}{(\lambda_2- \sqrt{\lambda_n})^2}$. Under the full rank breaking scheme that decomposes a $k$-way
comparison into $\binom{k}{2}$ pairwise comparisons, \prettyref{thm:FBUpperBound} gives an upper bound that scales as 
$\frac{ mk \log n}{\lambda_2^2}.$  If the comparison graph is an expander graph, i.e., $\lambda_2 \sim \lambda_n$ and $mk=\Omega(n\log n),$
our lower and upper bounds match up to a $\log n$ factor. 
This follows from the fact that $\sum_i \lambda_i =\sum_i d_i = mk$, and for expanders $mk=\Theta(n\lambda_2).$
Since the Erd\H{o}s-R\'enyi random graph is an expander graph with high probability for average degree larger than $\log n,$
when the system is allowed to choose the item assignment, we propose a random assignment scheme under which the items for each user are chosen {\em independently and uniformly at random}. 
It follows from \prettyref{thm:Oracle} that $mk =\Omega(n)$ is {\em necessary} for any item assignment scheme to reliably infer the underlying preference vector, while our upper bounds imply that $mk=\Omega(n \log n)$ is {\em sufficient} with the random assignment scheme and
can be achieved by either the  \ML estimator  or  the full rank breaking or the independence-preserving breaking that decompose a $k$-way comparison into
$\lfloor k/2 \rfloor$ non-intersecting pairwise comparisons,  proving that rank breaking schemes are also nearly optimal.

\subsection{Related Work}
There is a vast literature on rank aggregation, and here we can only hope to cover a fraction of them we see most relevant. 
In this paper, we study a statistical learning
approach, assuming the observed ranking data is generated from 
a probabilistic model. 
Various probabilistic models on permutations
have been studied in the ranking  literature (see, \eg, \cite{Liu10,Lozano12}).
A {nonparametric} approach to modeling distributions over rankings using sparse representations
has been studied in \cite{JS08}. 
Most of the {parametric} models fall into one of the following three categories:
noisy  comparison model, distance based model, and random utility model.
The noisy comparison model assumes that
there is an underlying true ranking over $n$ items, 
and each user independently gives a pairwise comparison 
which agrees with the true ranking with probability $p>1/2$.
It is shown in \cite{Mossel09} that $O(n \log n)$ pairwise comparisons, when chosen adaptively, are sufficient for accurately
estimating the true ranking.

The Mallows model is a distance-based model,
which randomly generates a full ranking $\sigma$ over $n$ items from some underlying true ranking $\sigma^\ast$ with probability
proportional to $e^{-\beta d(\sigma, \sigma^\ast)}$,
where $\beta$ is a fixed spread parameter and $d(\cdot,\cdot)$ can be any permutation distance such as the Kemeny distance.
 It is shown in \cite{Mossel09} that
the true ranking $\sigma^\ast$ can be estimated accurately given $O(\log n)$ independent full rankings generated under the Mallows
model with the Kemeny distance. 

In this paper, we study a special case of random utility models (RUMs) known as the Plackett-Luce (\PL) model. 
It is shown in \cite{hunter04} that the likelihood function under the \PL model is concave and the \ML estimator can be efficiently found using a
minorization-maximization (\MM) algorithm which is a variation of the general \EM algorithm.
We give an upper bound on the error achieved by such an \ML estimator, 
and prove that this is matched by a lower bound. 
The lower bound is derived by comparing to an oracle estimator which observes the random utilities of RUM directly. 
The Bradley-Terry (BT) model is the special case of the \PL model  where we only observe pairwise comparisons. 
For the BT model, \cite{Shah12} proposes RankCentrality algorithm based on
 the stationary distribution of a random walk over a suitably defined comparison graph and shows $\Omega(n {\sf poly}(\log n))$ randomly chosen pairwise comparisons are sufficient to accurately estimate the underlying parameters; 
one corollary of our result is a matching performance guarantee for the \ML estimator under the BT model.  More recently, \cite{Rajkumar14} analyzed various algorithms including RankCentrality and the \ML estimator under a general, not necessarily uniform, sampling scheme.

In a \PL model with priors, MAP inference becomes  computationally challenging.
Instead, an efficient message-passing algorithm is proposed in \cite{Guiver09} to approximate the MAP estimate.
For a more general family of random utility models, Soufiani et al. in \cite{azari_nips12,SPX13} give a
sufficient condition under which the likelihood function is concave,
and propose a Monte-Carlo \EM algorithm to compute the \ML estimator for general RUMs.
More recently in \cite{Souriani13,AzariSoufiani_icml14}, the generalized method of moments together with the rank-breaking is applied to estimate the parameters of the \PL model and the random utility model when the data consists of full rankings.


\section{Main results}
In this section, we present our theoretical findings and numerical experiments. 
\subsection{Oracle lower bound}
In this section, we derive an oracle lower bound for any estimator of $\theta^\ast$. The lower bound is constructed by considering an oracle who reveals all the hidden scores in the \PL model  as side information and holds for the general Thurstone models.
\begin{theorem} \label{thm:Oracle}
Suppose $\sigma_1^m$ are generated from the Thurstone model for some CDF $F.$   For any estimator $\widehat{\theta},$
 $$
\inf_{\widehat{\theta}}  \sup_{\theta^*\in \Theta_b} E[||\widehat{\theta} - \theta^*||^2]   \geq    \frac{1}{2I(\mu)+  \frac{2 \pi^2} {b^2 (d_1+d_2) } }  \sum_{i=2}^{n} \frac{1}{d_i} \ge  \frac{1}{2I(\mu)+  \frac{2 \pi^2} {b^2 (d_1+d_2) } }  \frac{(n-1)^2}{mk},
$$
where $\mu$ is the probability density function of $F$, i.e., $\mu=F'$ and $I(\mu)=\int \frac{ \left( \mu'(x) \right)^2 }{\mu(x)} dx$; 
the second inequality follows from the Jensen's inequality. For the \PL model, which is a special case of the Thurstone models with $F$ being the standard Gumbel distribution, $I(\mu)=1$.
\end{theorem}
\noindent
\prettyref{thm:Oracle} shows that the oracle lower bound scales as $ \sum_{i=2}^{n} \frac{1}{d_i}$.
We remark that the summation begins with $1/d_2.$  This makes some sense, in view of the fact that the parameters $\theta^\ast_i$ need to sum to zero.
For example, if $d_1$ is a moderate value and all the other $d_i$'s are very large, then we may be able to accurately estimate
$\theta_i^*$ for $i\neq 1$ and therefore accurately estimate $\theta_1^*.$   The oracle lower bound also depends on the dynamic range $b$ and is tight for $b=0$, because a trivial estimator that always outputs the all-zero vector achieves the lower bound.


\paragraph*{Comparison to previous work}
\prettyref{thm:Oracle} implies that $mk=\Omega(n)$ is necessary for any item assignment  scheme  to reliably infer $\theta^\ast$, i.e., ensuring $E[||\widehat{\theta} - \theta^*||^2]=o(n)$. It provides the first converse result on inferring the parameter vector under the general Thurstone models to our knowledge. For the Bradley-Terry model, which is a special case of the \PL model where all the partial rankings reduce to the pairwise comparisons, i.e., $k=2$, it is shown in \cite{Shah12} that $m=\Omega(n )$ is necessary for the random item assignment scheme to achieve the reliable inference based on the information-theoretic argument. In contrast, our converse result is derived based on the Bayesian Cram\'e-Rao lower bound \cite{gill1995}, applies to the general models with any item assignment, and  is considerably tighter if $d_i$'s are of different orders.

\subsection{Cram\'er-Rao lower bound}
In this section, we derive the Cram\'er-Rao lower bound for any unbiased estimator of $\theta^\ast$.
\begin{theorem} \label{thm:CramerRao}
Let $k_{\max}=\max_{j \in [m]} k_j$ and $\calU$ denote the set of all unbiased estimators of $\theta^\ast$, i.e., $\widehat{\theta} \in \calU$ if and only if $\mathbb{E}_\theta[\widehat{\theta}]=\theta, \forall \theta \in \Theta_b$. If $b>0$, then 
\begin{align*}
\inf_{\widehat{\theta} \in \calU} \sup_{\theta^\ast \in \Theta_b}  \mathbb{E}[ \| \widehat{\theta}-\theta^\ast \|^2 ] \ge \left(1-\frac{1}{k_{\max} } \sum_{\ell=1}^{k_{\max}} \frac{1}{\ell} \right)^{-1} \sum_{i=2}^n \frac{1}{\lambda_i} \ge \left(1-\frac{1}{k_{\max} } \sum_{\ell=1}^{k_{\max}} \frac{1}{\ell} \right)^{-1}\frac{(n-1)^2}{mk},
\end{align*}
where the second inequality follows from the Jensen's inequality.
\end{theorem}
The Cram\'er-Rao lower bound scales as $\sum_{i=2}^{n} \frac{1}{\lambda_i}$. When $G$ is disconnected, i.e., all the items can be partitioned into two groups such that no user ever compares an item in
one group with an item in the other group, $\lambda_2=0$ and the Cram\'er-Rao lower bound is infinity, which is valid (and of course tight) because there is no basis for gauging
any item in one connected component with respect to any item in the other connected component and the accurate inference is impossible for any estimator.
Although the Cram\'er-Rao lower bound only holds for any unbiased estimator, we suspect that a lower bound with the same scaling holds for any estimator, but
we do not have a proof.

\subsection{\ML upper bound}
In this section, we study the \ML estimator based on the partial rankings. The \ML estimator of $\theta^\ast$ is defined as
$\widehat{\theta}_{\ML} \in \arg \max_{\theta \in \Theta_b} \calL(\theta)$, where $\calL(\theta)$ is the log likelihood function given by
\begin{align}
\calL(\theta) = \log \mathbb{P}_{\theta}[\sigma_1^m ]= \sum_{j=1}^m  \sum_{\ell=1}^{k_j-1} \left[ \theta_{ \sigma_j(\ell) } - \log \left( \exp( \theta_{\sigma_j(\ell) }) + \cdots+ \exp(\theta_{\sigma_j (k_j)  } ) \right) \right]. \label{eq:loglikelihood}
\end{align}
As observed in \cite{hunter04}, $\calL(\theta)$ is concave in $\theta$ and thus the \ML estimator can be efficiently computed either via the gradient descent method or the \EM type algorithms.

The following theorem gives an upper bound on the error rates inversely dependent on $\lambda_2$. Intuitively, by the well-known Cheeger's inequality, if the spectral gap $\lambda_2$ becomes larger, then there are more edges across any bi-partition of $G$, meaning more pairwise comparisons are available between any bi-partition of movies, and therefore $\theta^\ast$ can be estimated more accurately.
\begin{theorem} \label{thm:MLEUpperBound}
Assume $\lambda_n \ge C \log n$ for a sufficiently large constant $C$ in the case with $k>2$. Then with high probability,
\begin{align*}
\|\widehat{\theta}_{\ML} - \theta^\ast \|_2 \le
\left\{
    \begin{array}{rl}
     4(1+e^{2b})^2 \lambda^{-1}_2  \sqrt{m\log n}  & \text{If } k=2, \\
       \frac{ 8 e^{4b} \sqrt{2mk \log n }  } {\lambda_2- 16 e^{2b} \sqrt{\lambda_n \log n}} & \text{If } k>2.
    \end{array} \right.
    \end{align*}
\end{theorem}
We compare the above upper bound with the Cram\'er-Rao lower bound given by \prettyref{thm:CramerRao}. Notice that $\sum_{i=1}^n \lambda_i=mk$ and $\lambda_1=0$. Therefore, $\frac{mk}{\lambda_2^2} \ge \sum_{i=2}^n \frac{1}{\lambda_i}$ and the upper bound is always larger than the Cram\'er-Rao lower bound. When the comparison graph $G$ is an expander and $mk =\Omega(n \log n)$, by the well-known Cheeger's inequality, $\lambda_2 \sim \lambda_n=\Omega(\log n)$ , the upper bound is only larger than the Cram\'er-Rao lower bound by a logarithmic factor.
In particular, with the random item assignment scheme,
we show that $\lambda_2, \lambda_n \sim \frac{mk}{n}$ if $m k \ge C \log n$ and as a corollary of \prettyref{thm:MLEUpperBound}, $mk=\Omega(n \log n )$ is sufficient to ensure $\| \widehat{\theta}_{\ML}-\theta^\ast \|_2 =o (\sqrt{n})$, proving the random item assignment scheme with the \ML estimation is minimax-optimal up to a $\log n$ factor.

\begin{corollary}\label{cor:MLEUpperBound}
Suppose $S_1^m$ are chosen independently and uniformly at random among all possible subsets of $[n]$.
Then there exists a positive constant $C>0$ such that if
 $m \ge C n \log n$ when $k=2$ and $m k \ge Ce^{2b} \log n$ when $k>2$, then with high probability
\begin{align*}
\| \widehat{\theta}_{\ML}-\theta^\ast \|_2  \le
\left\{
    \begin{array}{rl}
    4 (1+e^{2b})^2  \sqrt{ \frac{ n^2\log n}{m}} \;,\;\; & \text{if } k=2, \\
     32 e^{4b} \sqrt{ \frac{2n^2\log n}{mk} }  \;,\;\; & \text{if } k>2.
    \end{array} \right.
\end{align*}
\end{corollary}

\paragraph*{Comparison to previous work} \prettyref{thm:MLEUpperBound} provides the first finite-sample error rates for inferring the parameter vector under the \PL model to our knowledge. For the Bradley-Terry model, which is a
special case of the \PL model with $k=2$, \cite{Shah12} derived the similar performance guarantee by analyzing the rank centrality algorithm and the \ML estimator.
 More recently, \cite{Rajkumar14} extended the results to the non-uniform sampling scheme of item pairs, but the performance guarantees  obtained when specialized to the uniform sampling scheme require at least $m=\Omega(n^4 \log n)$ to ensure $\| \widehat{\theta}-\theta^\ast \|_2 =o (\sqrt{n})$, while our results only require $m=\Omega(n \log n)$.

\subsection{Rank breaking upper bound}
In this section, we study two rank-breaking schemes which decompose partial rankings into pairwise comparisons.
\begin{definition}
Given a partial ranking $\sigma$ over the subset $S \subset [n]$ of size $k$, the independence-preserving breaking scheme (\IB)  breaks $\sigma$ into  $\lfloor k/2 \rfloor$ non-intersecting pairwise comparisons of form $\{i_t, i'_t, y_t\}_{t=1}^{\lfloor k/2 \rfloor}$ such that
$\{i_s, i'_s\} \cap \{i_t, i'_t\}=\emptyset$ for any $s\neq t$ and $y_t=1$ if  $\sigma^{-1}(i_t)< \sigma^{-1}(i'_t)$ and $0$ otherwise. The random \IB chooses $\{i_t, i'_t\}_{t=1}^{\lfloor k/2 \rfloor}$ uniformly at random among all possibilities.
\end{definition}
If $\sigma$ is generated under the $\PL$ model, then the \IB breaks $\sigma$ into independent pairwise comparisons generated under the $\PL$ model. Hence, we can first break partial rankings $\sigma_1^m$ into independent pairwise comparisons using the random \IB and then apply the \ML estimator on the generated pairwise comparisons with the constraint that $\theta \in \Theta_b$, denoted by $\widehat{\theta}_{\IB}$.
Under the random assignment scheme, as a corollary of \prettyref{thm:MLEUpperBound}, $mk=\Omega(n \log n )$ is sufficient to ensure $\| \widehat{\theta}_{\IB}-\theta^\ast \|_2 =o (\sqrt{n})$, proving the random item assignment scheme with the  random \IB is minimax-optimal up to a $\log n$ factor in view of the oracle lower bound in \prettyref{thm:Oracle}.
\begin{corollary} \label{cor:IBUpperBound}
Suppose $S_1^m$ are chosen independently and uniformly at random among all possible subsets of $[n]$ with size $k.$
There exists a positive constant $C>0$ such that if $mk \ge C  n \log n$, then with high probability,
\begin{align*}
\| \widehat{\theta}_{\IB}-\theta^\ast \|_2  \le 4 (1+e^{2b})^2  \sqrt{ \frac{ 2n^2\log n}{mk }}.
\end{align*}
\end{corollary}

\begin{definition}
Given a partial ranking $\sigma$ over the subset $S \subset [n]$ of size $k$, the full breaking scheme (\FB)  breaks $\sigma$ into  all $\binom{k}{2}$ possible pairwise comparisons of form $\{i_t, i'_t, y_t\}_{t=1}^{\binom{k}{2}}$ such that
$y_t=1$ if  $\sigma^{-1}(i_t)< \sigma^{-1}(i'_t)$ and $0$ otherwise.
\end{definition}
If $\sigma$ is generated under the $\PL$ model, then the \FB breaks $\sigma$ into pairwise comparisons which are not independently generated under the $\PL$ model.
We pretend the pairwise comparisons induced from the full breaking are all independent and maximize the weighted log likelihood function given by
\begin{align}
\calL(\theta)= \sum_{j=1}^m \frac{1}{2(k_j-1)} \sum_{i,i' \in S_j}  \left( \theta_i \1{\sigma_j^{-1}(i)<\sigma_j^{-1}(i')} + \theta_{i'} \1{\sigma_j^{-1}(i)>\sigma_j^{-1}(i')}  - \log \left(e^{\theta_{i}}+e^{\theta_{i'}}\right) \right) \label{eq:loglikelihoodbreaking}
\end{align}
with the constraint that $\theta \in \Theta_b$.
Let $\widehat{\theta}_{\FB}$ denote the maximizer. Notice that we put the weight $\frac{1}{k_j-1}$ to adjust the contributions of the pairwise comparisons generated from the partial rankings over subsets with different sizes.
\begin{theorem} \label{thm:FBUpperBound}
With high probability,
\begin{align*}
\| \widehat{\theta}_{\FB}-\theta^\ast \|_2  \le 2(1+e^{2b})^2 \frac{\sqrt{mk \log n} } {\lambda_2}  .
\end{align*}
Furthermore, suppose $S_1^m$ are chosen independently and uniformly at random among all possible subsets of $[n]$. There exists a positive constant $C>0$ such that if $mk \ge C  n \log n$, then with high probability,
\begin{align*}
\| \widehat{\theta}_{\FB}-\theta^\ast \|_2  \le 4 (1+e^{2b})^2 \sqrt{  \frac{n^2 \log n} {mk}}.
\end{align*}
\end{theorem}
\prettyref{thm:FBUpperBound} shows that the error rates of $\widehat{\theta}_{\FB}$ inversely depend on $\lambda_2$. When the comparison graph $G$ is an expander, i.e., $\lambda_2 \sim \lambda_n$, the upper bound is only larger than the Cram\'er-Rao lower bound by a logarithmic factor. The similar observation holds for the \ML estimator as shown in  \prettyref{thm:MLEUpperBound}.
With the random item assignment scheme,  \prettyref{thm:FBUpperBound} imply that the \FB only need $mk = \Omega(n \log n)$ to achieve the reliable inference, which is optimal up to a $\log n$ factor in view of the oracle lower bound in \prettyref{thm:Oracle}.

\paragraph*{Comparison to previous work}  The rank breaking schemes considered in \cite{Souriani13,AzariSoufiani_icml14} breaks the full rankings
according to rank positions while our schemes break the partial rankings according to the item indices. The results in \cite{Souriani13,AzariSoufiani_icml14}
establish the consistency of the generalized method of moments under the rank breaking schemes when the data consists of full rankings. In contrast, \prettyref{cor:IBUpperBound} and
\prettyref{thm:FBUpperBound} apply to the more general setting with partial rankings and provide the finite-sample error rates, proving the optimality of the random \IB and \FB with the random item assignment scheme. 

\subsection{Numerical experiments}
Suppose there are $n=1024$ items and the underlying preference vector $\theta^\ast$ is uniformly distributed over $[-b, b]$.
We generate $d$ full rankings over $1024$ items according to the $\PL$ model with parameter $\theta^\ast.$
Fix a $k \in \{ 512, 256,\ldots, 2 \}$. We break each full ranking $\sigma$ into $n/k$ partial rankings over subsets of size $k$ as follows: Let $\{S_j\}_{j=1}^{n/k}$ denote a partition of $[n]$ generated uniformly at random such that $S_j \cap S_{j'} = \emptyset$ for $j \neq j'$ and $|S_j|=k$ for all $j$; generate $\{\sigma_j\}_{j=1}^{n/k}$ such that $\sigma_j$ is the partial ranking over set $S_j$ consistent with $\sigma$.
In this way, in total we generate $m=dn/k$ $k$-way comparisons which are all independently generated from the \PL model. To compute the \ML estimator of $\theta^\ast$ based on the generated partial rankings, we apply the minorization-maximation (MM) algorithm proposed in \cite{hunter04}. We measure the estimation error by the normalized mean square error (MSE) defined as $ \frac{mk}{n^2} \| \widehat{\theta}_{\ML} -\theta^\ast\|^2$.



We run the simulation with $b=0, 2$ and $d=16, 32, 64,128$. The results are depicted in Fig.~\ref{fig:varying1024}.
We also plot the Cram\'er-Rao limit given by $\left(1-\frac{1}{k} \sum_{l=1}^k \frac{1}{l} \right)^{-1}$ as per \prettyref{thm:CramerRao}. The oracle lower bound in \prettyref{thm:Oracle}
implies that the normalized MSE is at least $1$.
We can see that the normalized MSE approaches the Cram\'er-Rao limit as $d$ increases and achieves the oracle lower bound if further $k$ becomes large, suggesting the \ML estimator is minimax-optimal. Moreover, with a large number of partial rankings available, i.e., $d$ is large enough, when $k$ is decreased from $n$ to $2$, 
the normalized MSE increases roughly by a factor of $4$ if $b=0$ and $6$ if $b=2$, suggesting that the random \IB is minimax-optimal up to a $\log n$ factor.  
Also, we observe that the normalized MSE is not as sensitive to the value of $b$ as claimed by our upper bounds given by \prettyref{cor:MLEUpperBound}. Notice that in the case with $b=2$, according to the \PL model, the item with the highest preference is ranked higher than the item with lowest preference with probability $\frac{e^{4}}{1+e^{4}} \approx 0.98$.

%
\begin{figure}[ht]
\centering
\subfigure{
\includegraphics[width=3in]{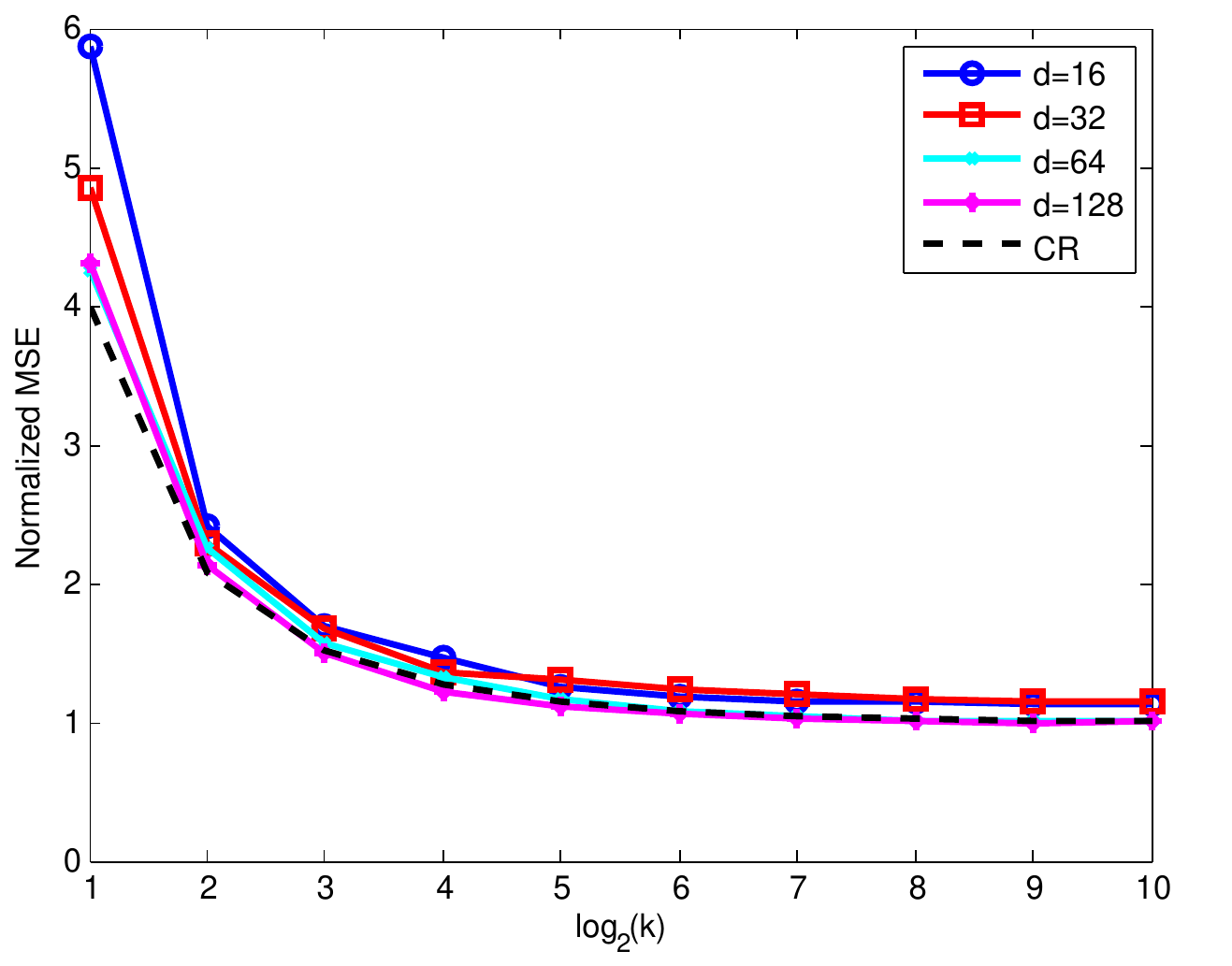}
\label{fig:varying1024b0}
}
\quad
\subfigure{
\includegraphics[width=3in]{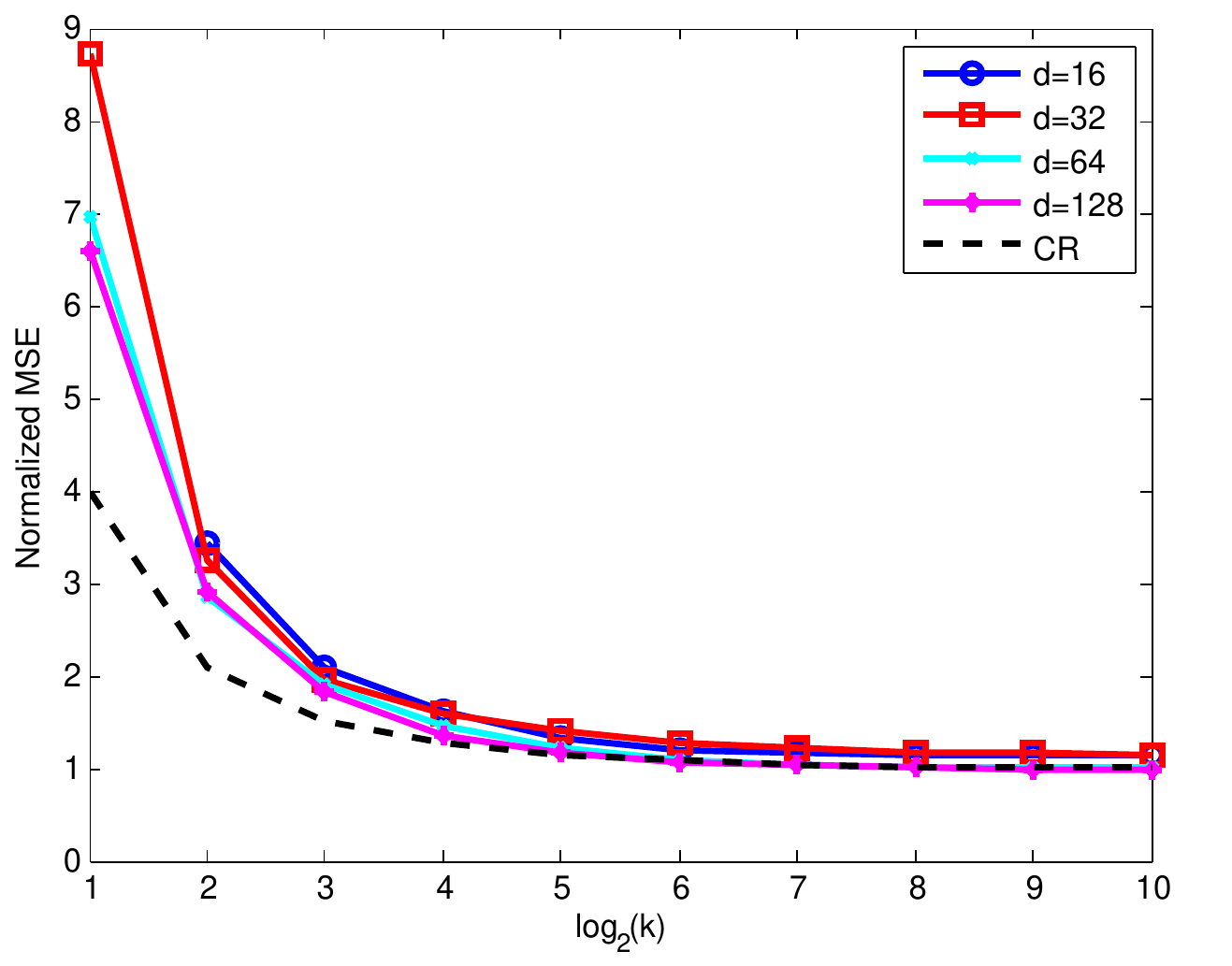}
\label{fig:varying1024b2}}
\caption{The \ML estimator based on $nd/k$ $k$-way comparisons: (a):$b=0$; (b):$b=2$.}
\label{fig:varying1024}
\end{figure}
%
%
%

\section{Proofs}
We introduce some additional notations used in the proof.
For a vector  $x$, let $\|x\|_2$ denote the usual $l_2$ norm.
Let $\mathbf{1}$ denote the all-one vector and $\mathbf{0}$ denote the all-zero vector with the appropriate dimension.
Let $\calS^n$ denote the set of $n \times n$ symmetric matrices with real-valued entries.
 For $X \in \calS^n$, let $\lambda_1(X) \le \lambda_2(X) \le \cdots  \le \lambda_n(X)$ denote its eigenvalues sorted in increasing order. Let $\Tr(X)=\sum_{i=1}^n \lambda_i(X) $ denote its trace and $\|X\|=\max\{ - \lambda_1(X), \lambda_n(X)\}$ denote its spectral norm.
For two matrices $X, Y \in \calS^n$, we write $X\le Y$ if $Y-X$ is positive semi-definite, i.e., $\lambda_1(Y-X) \ge 0$. 
Recall that $\calL (\theta)$ is the log likelihood function.
The first-order partial derivative of $\calL(\theta)$ is given by
\begin{align}
\nabla_i \calL(\theta) = \sum_{j: i \in S_j} \sum_{\ell=1}^{k_j-1}  \1{\sigma_j^{-1}(i) \ge \ell}  \left[ \1{\sigma_{j}(\ell)=i} -  \frac{\exp(\theta_i) }{ \exp( \theta_{\sigma_j(\ell) }) + \cdots+ \exp(\theta_{\sigma_j (k_j)  } )} \right], \forall i \in [n] \label{eq:gradient}
\end{align}
and the Hessian matrix $H(\theta) \in \calS^n$ with $H_{ii'}(\theta)=\frac{\partial^2 \calL(\theta)}{\partial \theta_i \partial \theta_{i'}}$ is given by
\begin{align}
H(\theta)=-\frac{1}{2} \sum_{j=1}^m \sum_{i,i' \in S_j} (e_i-e_{i'} ) (e_i - e_{i'} )^\top \sum_{\ell=1}^{k_j-1}  \frac{\exp(\theta_i + \theta_{i'})  \1{\sigma_j^{-1}(i), \sigma_j^{-1}(i') \ge \ell}   }{ [ \exp( \theta_{\sigma_j(\ell) }) + \cdots+ \exp(\theta_{\sigma_j (k_j)  } ) ]^2 }. \label{eq:Hessian}
\end{align}
It follows from the definition that $-H(\theta)$ is positive semi-definite for any $\theta \in \reals^n$. Define $L_j \in S^n$ as
\begin{align*}
L_j = \frac{1}{2(k_j-1)}  \sum_{i,i' \in S_j}  (e_i-e_{i'} ) (e_i - e_{i'} )^\top,
\end{align*}
and then the Laplacian of the pairwise comparison graph $G$ satisfies $L=\sum_{j=1}^m L_j$.

 \subsection{Proof of \prettyref{thm:Oracle}}
We first introduce a  key auxiliary result used in the proof. Let $F$ be a fixed CDF (to be used in the Thurstone model),
let $b >0$ and suppose $\theta$ is a parameter to be estimated with $\theta \in [-b,b]$
from observation $U=(U_1, \ldots  , U_d),$   where the $U_i$'s are independent with the common CDF given by $F(c-\theta)$.
The following proposition gives a lower bound on the average MSE for a fixed prior distribution based on Van Trees inequality \cite{gill1995}.
\begin{proposition}\label{prop:CR_exp1}
Let $p_0$ be a probability density on $[-1,1]$ such that $p_0(1)=p_0(-1)=0$ and define the prior density of $\Theta$ as $p (\theta)=\frac{1}{b} p_0( \frac{\theta}{b})$.
Then for any estimator $T(U)$ of  $\Theta$,
$$
E[(\Theta - T(U))^2] \geq \frac{1}{d} \frac{1}{ I(\mu)+ I(p_0)/(b^2 d)},
$$
where $\mu$ is the probability density function of $F$ with $I(\mu)=\int \frac{ \left( \mu'(x) \right)^2 }{\mu(x)} dx$ and $I(p_0)=\int_{-1}^1 \frac{ \left( p_0'(\theta) \right)^2 }{p_0(\theta)} d\theta$. 
\end{proposition}
\begin{proof}
It follows from the Van Trees inequality that
\begin{align*}
E[(\Theta - T(U))^2]   \geq \frac{1}{ \int I(\theta) p(\theta) d\theta  + I (p)},
\end{align*}
where the Fisher information $I(\theta)=dI(\mu)$ and
\begin{align*}
I(p)= \int_{-b}^b \frac{ \left( p'(\theta) \right)^2 }{p(\theta)} d\theta = \frac{1}{b^2} \int_{-1}^1 \frac{ \left( p_0'(\theta) \right)^2 }{p_0(\theta)} d\theta = \frac{1}{b^2} I(\lambda_0).
\end{align*}

\end{proof}

\begin{proof}[Proof of \prettyref{thm:Oracle}]
Let $\widehat{\theta}$ be a given estimator.
The minimax MSE for $\widehat{\theta}$  is greater than or equal to the average MSE for a given prior distribution on
$\theta^*.$     Let $p_0(\theta) =  \cos^2(\pi \theta/2)$,  then $I(p_0)=\pi^2$.
Define $p (\theta)=\frac{1}{b} p_0( \frac{\theta}{b})$.
If $n$ is even we use the following prior distribution.  The prior distribution of $\theta^*_i$ for $i$ odd is $p(\theta)$ and for $i$ even, $\theta^*_i \equiv -\theta^*_{i-1}.$      If $n$ is odd use the same distribution for
$\theta^*_1$ through $\theta^*_{n-1}$ and set $\theta^*_n\equiv 0.$   Note that $\theta^* \in \Theta_b$ with probability one.
For simplicity, we assume $n$ is odd in the rest of this proof; the modification for $n$ even is trivial.
We use the genie argument, so that the observer can see the hidden utilities in the Thurstone model.
The estimation of $\theta^*$ decouples into $\lfloor \frac{n}{2} \rfloor$
disjoint problems, so we can focus on the estimation of $\theta_1$ from the vector of random variables $U=(U_1, \ldots , U_{d_1})$
associated with item 1 and the vector of random variables $V=(V_1, \ldots , V_{d_2})$ associated with item 2.  The distribution functions of the $U_i$'s are
all $F(c-\theta_1^\ast)$ and the distribution functions of the $V_i$'s are all $F(c + \theta_1^\ast)$, and
the $U$'s and $V$'s are all mutually independent given $\theta^*.$  Recall that $\mu$ is the probability density function of $F$, i.e., $\mu=F'$.
The Fisher information for each of the $d_1+d_2$ observations is $I(\mu)$,
so that \prettyref{prop:CR_exp1} carries over to this situation with $d=d_1+d_2.$   Therefore, for any estimator $T(U,V)$ of $\Theta^*_1$  (the random
version of $\theta_1^*$),
$$
E[(\Theta_1^* - T(U,V))^2] \geq \frac{1}{d_1+d_2} \frac{1}{I(\mu)+\pi^2/(b^2(d_1+d_2)) }
$$
By this reasoning, for any odd value of $i$ with $1 \leq i <   n$  we have
\begin{align*}
E[(\widehat{\theta}_i- \theta^*_i)^2]  +  E[(\widehat{\theta}_{i+1} - \theta^*_{i+1})^2] & \geq \frac{2}{I(\mu)+\pi^2/(b^2 (d_1+d_2)  ) } \frac{1}{d_{i}+d_{i+1}} \\
& \geq  \frac{1}{2I(\mu)+2 \pi^2/(b^2 (d_1+d_2) ) } \left(\frac{1}{d_{i+1}}  + \frac{1}{d_{i+2}}\right).
\end{align*}
Summing over all odd values of $i$ in the range $1\leq i <n$  yields the theorem. Furthermore, since $\sum_{i=1}^n d_i=mk$, by Jensen's inequality,
$
\sum_{i=2}^{n} \frac{1}{d_i}  \ge \frac{(n-1)^2}{\sum_{i=2}^n d_i} \ge \frac{(n-1)^2}{mk}.
$
\end{proof}

\subsection{Proof of \prettyref{thm:CramerRao}}
The Fisher information matrix is defined as $I(\theta) = - \mathbb{E}_{\theta} [ H(\theta)]$ and given by
\begin{align*}
I(\theta)=\frac{1}{2} \sum_{j=1}^m \sum_{i,i' \in S_j} (e_i-e_{i'} ) (e_i - e_{i'} )^\top \sum_{l=1}^{k_j-1} \mathbb{P}_{\theta}[ \sigma_j^{-1}(i), \sigma_j^{-1}(i') \ge \ell ]  \frac{e^{ \theta_i+\theta_{i'}} }{ [e^{\theta_{\sigma_j(\ell) }} + \cdots+ e^{\theta_{\sigma_j (k_j)  } } ]^2 }.
\end{align*}
Since $-H(\theta)$ is positive semi-definite, it follows that $I(\theta)$ is positive semi-definite. Moreover, $\lambda_1(I(\theta))$ is zero and the corresponding eigenvector is the normalized all-one vector. Fix any unbiased estimator $\widehat{\theta} $ of $\theta \in \Theta_b$. Since $\widehat{\theta} \in \calU$, $\widehat{\theta}-\theta$ is orthogonal to $\mathbf{1}$. The Cram\'er-Rao lower bound then implies that
$
\mathbb{E}[ \| \widehat{\theta} -\theta \|^2 ] \ge \sum_{i=2}^n \frac{1}{\lambda_i (I(\theta) )}.
$
Taking the supremum over both sides gives
\begin{align*}
 \sup_{\theta} \mathbb{E}[ \| \hat{\theta} -\theta \|^2 ] \ge \sup_{\theta} \sum_{i=2}^n \frac{1}{\lambda_i (I(\theta) )} \ge   \sum_{i=2}^n \frac{1}{\lambda_i (I(0) )}.
\end{align*}
If $\theta$ equals the all-zero vector, then
\begin{align*}
\mathbb{P}[ \sigma_j^{-1}(i), \sigma_j^{-1}(i') \ge \ell] = \frac{(k_j-2)(k_j-3) \cdots (k_j- \ell) }{k_j(k_j-1) \cdots (k_j-\ell+2) } = \frac{(k_j-\ell+1)(k_j-\ell)}{k_j(k_j-1)}.
\end{align*}
It follows from the definition that
\begin{align*}
 I(0) = \frac{1}{2} \sum_{j=1}^m \sum_{i,i' \in S_j} (e_i-e_{i'} ) (e_i - e_{i'} )^\top \sum_{l=1}^{k_j-1} \frac{k_j-\ell}{k_j (k_j-1) (k_j- \ell+1)}
 \le  \left(1-\frac{1}{k_{\max} } \sum_{\ell=1}^{k_{\max}} \frac{1}{\ell } \right) L.
\end{align*}
By Jensen's inequality,
\begin{align*}
\sum_{i=2}^n \frac{1}{\lambda_i} \ge \frac{(n-1)^2}{\sum_{i=2}^n \lambda_i} = \frac{(n-1)^2}{\Tr(L)} =  \frac{(n-1)^2}{\sum_{i=1}^n d_i}=\frac{(n-1)^2}{mk}. 
\end{align*}

\subsection{Proof of \prettyref{thm:MLEUpperBound}}
The main idea of the proof is inspired from the proof of \cite[Theorem 4]{Shah12}. We first introduce several key auxiliary results used in the proof. Observe that $ \Expect_{\theta^\ast} [\nabla L(\theta^\ast)] =0$. The following lemma
upper bounds the deviation of $\nabla L(\theta^\ast)$ from its mean.

\begin{lemma} \label{lmm:Gradient}
With probability at least $1-\frac{2e^2}{n}$,
\begin{align}
\|\nabla {\cal L}(\theta^\ast)\|_2 \le   \sqrt{2mk\log n}    \label{eq:gradientbound}
\end{align}
\end{lemma}

\begin{proof}
The idea of the proof is to view $\nabla {\cal L}(\theta^\ast)$ as the final value of a discrete time vector-valued
martingale with values in $\IR^n.$   Consider a user that ranks
items $1, \ldots , k.$     The PL model for the ranking can be generated in a series of $k-1$ rounds.   In the first
round, the top rated item for the user is found.  Suppose it is item $I$.   This contributes the term $e_I - (p_1, p_2 , \ldots , p_k, 0 , 0, \ldots  , 0)$
to $\nabla {\cal L}(\theta^\ast),$  where $p_i=P\{I=i\}.$   This contribution is a mean zero random vector in
$\IR^n$ and its norm is less than one.   For notational convenience, suppose $I=k.$     In the second round, item $k$
is removed from the competition, and an item $J$ is to be selected at random from
among $\{1, \ldots , k-1\}.$   If $q_j$ denotes $P\{J=j\}$ for $1\leq j \leq k-1,$
then the contribution of the second round for the user to $\nabla {\cal L}(\theta^\ast)$ is the random vector
$e_J - (q_1, q_2 , \ldots , q_{k-1}, 0 , 0, \ldots  , 0),$
which has conditional mean zero (given $I$) and norm less than or equal to one.
Considering all $m$ users and $k_j-1$ rounds for user $j$, we see that $\nabla {\cal L}(\theta^\ast)$ is the value
of a discrete-time martingale at time $m(k-1)$ such that the martingale has initial value zero and  increments with
norm bounded by one.   By the vector version of the Azuma-Hoeffding inequality found in \cite[Theorem 1.8]{Hayes2005}
we have
$$
\mathbb{P}\{  \| \nabla {\cal L}(\theta^\ast) \|  \geq \delta  \}  \leq  2e^2 e^{-\frac{\delta^2}{2m(k-1)}},
$$
which implies the result.
\end{proof}

Observed that $ -H(\theta)$ is positive semi-definite with the smallest eigenvalue equal to zero. The following lemma
lower bounds its second smallest eigenvalue.

\begin{lemma} \label{lmm:Hessian}
Fix any $\theta \in \Theta_b$. Then
\begin{align}
\lambda_2\left(-H(\theta)\right) \ge
\left\{
    \begin{array}{rl}
    \frac{e^{2b}}{(1+e^{2b})^2} \lambda_2 & \text{If } k=2, \\
     \frac{1}{4e^{4b}} \left( \lambda_2 - 16 e^{2b} \sqrt{\lambda_n \log n} \right) & \text{If } k>2 ,
    \end{array} \right.
\end{align}
where the inequality holds with probability at least $1-n^{-1}$ in the case with $k>2$.
\end{lemma}
\begin{proof}
{\bf Case $k_j=2,\forall j \in [m]$}: The Hessian matrix simplifies as
\begin{align*}
H(\theta)=-\frac{1}{2} \sum_{j=1}^m  \sum_{i,i' \in S_j}  (e_i-e_{i'} ) (e_i - e_{i'} )^\top \frac{\exp(\theta_i)  }{ \exp( \theta_{i}) +  \exp(\theta_{i'} ) }
\frac{\exp(\theta_{i'})  }{ \exp( \theta_{i}) +  \exp(\theta_{i'} ) }.
\end{align*}
Observe that $H(\theta)$ is deterministic given $S_1^m$. Since $|\theta_i | \le b, \forall i \in [n]$,
\begin{align*}
\frac{\exp(\theta_i)\exp(\theta_{i'} ) }{ \left[ \exp( \theta_{i}) +  \exp(\theta_{i'} ) \right]^2 } \ge \frac{e^{2b}}{(1+e^{2b})^2}.
\end{align*}
It follows that $-H(\theta) \ge \frac{e^{2b}}{(1+e^{2b})^2} L$ and the theorem follows.

\noindent {\bf Case $k_j>2$ for some $j \in [m] $}: 
We first introduce a key auxiliary result used in the proof. 
\begin{claim}   \label{clm:ev_bound} Given $\theta \in   \IR^r,$  let  $ A =\mbox{diag}{(p)} - pp^T,$ where
$p$ is the column probability vector with \\ $p_i=e^{\theta_i}/(e^{\theta_1}+\cdots + e^{\theta_r})$ for each $i.$
If  $| \theta_i | \leq b,$ for $1\leq i \leq r,$  then $\lambda_2(A) \geq \frac{1}{re^{2b}}.$  Equivalently,
$e^{2b}A \geq B$ where $B= \frac{1}{r} \mbox{diag}(\mathbf{1}) - \frac{1}{r^2}\mathbf{1}\mathbf{1}^\top  .$
\end{claim}
 \begin{proof}  Fix $\theta$ satisfying the conditions of the lemma.  It is easy to see that
 for each $i,$  $p_i \geq \frac{1}{re^{2b}}.$    The matrix $A$ is
 positive semidefinite, and its smallest eigenvalue is zero, with the corresponding eigenvector $\mathbf{1}.$
 So $\lambda_2(A) = \min_{\alpha} \alpha^T A \alpha$ subject to the constraints
 $\alpha^T \mathbf{1} = 0$ and $\| \alpha \|^2=1.$   For $\alpha$ satisfying the constraints,
\begin{eqnarray*}
 \alpha^T A \alpha   & = &   \sum_i \alpha_i^2 p_i  - \left(\sum_j \alpha_ j p_j \right)^2 =   \sum_i    \left(  \alpha_i   -  \sum_j \alpha_j p_j   \right)^2  p_i  \\
 & = &  \min_c \sum_{i=1}^r (\alpha_i-c)^2 p_i  \geq  \min_c \sum_{i=1}^r (\alpha_i-c)^2  \frac{1}{re^{2b}}  \\
& = & \sum_{i=1}^r   \alpha_i^2    \frac{1}{re^{2b}}    =    \frac{1}{re^{2b}}
\end{eqnarray*}
The proof of the first part of the lemma is complete.   We remark that the bound of the lemma is nearly tight for
the case $\theta_1= \ldots = \theta_{r-1}=b$ and $\theta_r=-b,$
for which $\lambda_2(A) = \frac{ e^{2b}r}{((r-1)e^{2b}+1)^2} .$
The final equivalence mentioned in the lemma follows from the facts
 $\lambda_1(e^{2b}A)=\lambda_1(B)=0$ with common corresponding eigenvector $\mathbf{1},$  and
 $\lambda_i(e^{2b}A)\geq \frac{1}{r} = \lambda_i(B)$ for $2\leq i \leq r.$
\end{proof}
The Hessian matrix $H(\theta)$ depends on $\sigma_1^m$ and therefore is random given $S_1^m$. 
For a given user $j,$ and $\ell$ with $1\leq \ell \leq k_j-1,$  let $S^{(j,\ell)}$ denote the set of items
contending for the $\ell^{th}$ position in the ranking of user $j$ after higher ranking items have been selected:
$S^{(j,\ell)} =  \{ i :  \sigma_j^{-1}(i) \geq \ell \},$
let $\mathbf{1}^{(j,\ell)}$ denote the indicator vector for the set $S^{(j,\ell)},$
and let $p^{(j,\ell)}$ denote the corresponding probability column vector for the selection:
$$
p^{(j,\ell)}_i = P( \sigma_{j}(\ell) = i  | \sigma_j(1), \ldots  , \sigma_j(\ell-1) )
=  \frac{\mathbf{1}^{(j,\ell)}_i e^{\theta_i}}{\sum_{i' \in S_{j,\ell}}  e^{\theta_{i'}}}
 $$
The Hessian can be written as
$
H(\theta) = \sum_{j=1}^m  \sum_{\ell=1}^{k_j-1}  H^{(j,\ell)}
$
where
$$
- H^{(j,\ell)}  =\frac{1}{2}  \sum_{i, i' \in S^{(j,\ell)} }   (e_i-e_{i'} ) (e_i - e_{i'} )^\top   p^{(j,\ell)}_ip^{(j,\ell)}_{i'} =   \mbox{diag}(p^{(j,\ell)})-p^{(j,\ell)}(p^{(j,\ell)})^\top
$$
By \prettyref{clm:ev_bound} applied to the restriction of  $-H^{(j,\ell)}$ to $S^{(j,\ell)}\times S^{(j,\ell)},$
\begin{eqnarray}
-e^{2b}H^{(j,\ell)} & \geq  &
\frac{1}{k_j-\ell+1} \mbox{diag}(\mathbf{1}^{(j,\ell)})- \frac{1}{(k_j-\ell+1)^2}\mathbf{1}^{(j,\ell)}(\mathbf{1}^{(j,\ell)})^\top \nonumber  \\
&  =  & \frac{1}{2(k_j-\ell+1)^2}  \sum_{i, i' \in S^{(j,\ell)} }   (e_i-e_{i'} ) (e_i - e_{i'} )^\top  \label{eq.Hjl}
\end{eqnarray}
Summing over $j$ and $\ell$ in \eqref{eq.Hjl} and noting that $k_j-\ell+1 \leq k_j$ for all $j,\ell$ yields
\begin{align}
- e^{2b} H (\theta) \ge \frac{1}{2 }  \sum_{j=1}^m \sum_{i,i' \in S_j}   (e_i-e_{i'} ) (e_i - e_{i'} )^\top  \frac{1}{k_j^2}  \sum_{\ell=1}^{k_j-1} \1{\sigma_j^{-1}(i), \sigma_j^{-1}(i') \ge \ell}  :=\tilde{L} \label{eq:HessianLowerbound}
\end{align}
Observe that
\begin{align*}
\sum_{\ell=1}^{k_j-1} \mathbb{P}_{\theta} \left[ \sigma_j^{-1}(i), \sigma_j^{-1}(i') \ge \ell \right]  = 1+ \sum_{i'' \in S_j} \1{i'' \neq i,i'} \frac{e^{\theta_{i''}} }{e^{\theta_{i}} +e^{\theta_{i'}} +e^{\theta_{i''}}  } \ge 1+ \frac{k_j-2}{2e^{2b}+1} \ge \frac{k_j+1}{3e^{2b}}.
\end{align*}
Recall that $L$ is the Laplacian of $G$ and $L=\sum_{j=1}^m L_j$. It follows that
\begin{align}
 \mathbb{E}_{\theta}[ \tilde{L}] &=\frac{1}{2}   \sum_{j=1}^m  \sum_{i,i' \in S_j}   (e_i-e_{i'} ) (e_i - e_{i'} )^\top \frac{1}{k_j^2} \sum_{\ell=1}^{k_j-1} \mathbb{P}_{\theta}[ \sigma_j^{-1}(i), \sigma_j^{-1}(i') \ge \ell] \nonumber \\
 & \ge \frac{1}{2}   \sum_{j=1}^m  \sum_{i,i' \in S_j}   (e_i-e_{i'} ) (e_i - e_{i'} )^\top \frac{k_j+1}{3e^{2b} k_j^2} \nonumber \\
 &\ge \frac{1}{2}   \sum_{j=1}^m  \sum_{i,i' \in S_j}   (e_i-e_{i'} ) (e_i - e_{i'} )^\top \frac{1}{4e^{2b} (k_j-1) } =\frac{1}{4e^{2b} } L \label{eq:boundLaplacian}
\end{align}
Define $a_{ii'}=  \frac{1}{k_j^2} \sum_{\ell=1}^{k_j-1} \left( \1{\sigma_j^{-1}(i), \sigma_j^{-1}(i') \ge \ell } - \mathbb{P}_{\theta} [ \sigma_j^{-1}(i), \sigma_j^{-1}(i') \ge \ell] \right)$. Then
\begin{align*}
\tilde{L}- \mathbb{E}_{\theta}[ \tilde{L}] &= \frac{1}{2} \sum_{j=1}^m \left( \sum_{i,i' \in S_j}  a_{ii'} (e_i-e_{i'} ) (e_i - e_{i'} )^\top  \right) := \sum_{j=1}^m Y_j.
\end{align*}
Observe that $|a_{ii'}| \le \frac{1}{k_j}$ and therefore
$
-  \frac{(k_j-1)}{k_j} L_j \le Y_j \le \frac{(k_j-1)}{k_j}  L_j.
$
Furthermore, $\|L_j\|= \frac{k_j}{k_j-1}$ and thus $\|Y_j\| \le  1$. Moreover,
$
Y_j^2=  \sum_{i,i',i'' \in S_j}  a_{ii'} a_{ii''} (e_i-e_{i'} ) (e_i - e_{i''} )^\top .
$
It follows that for any vector $x \in \mathbb{R}^n$,
\begin{align*}
x^\top Y_j^2 x &=  \sum_{i,i',i'' \in S_j}  a_{ii'} a_{ii''} (x_i-x_{i'} ) (x_i - x_{i''}) \le \frac{1}{k_j^2}  \sum_{i,i',i'' \in S_j} |x_i-x_{i'}| |x_i - x_{i''} | \\
& = \frac{1}{k_j^2} \sum_{i \in S_j } \left( \sum_{i' \in S_j}  |x_i-x_{i'}| \right)^2 \le \frac{1}{k_j} \sum_{i,i' \in S_j} (x_i-x_{i'} )^2 = 2 x^\top L_j x,
\end{align*}
where the last inequality follows from the Cauchy-Swartz inequality. Therefore, $Y_j^2 \le 2 L_j$. It follows that $\sum_{j=1}^m \mathbb{E}_{\theta}[ Y_j^2  ] \le 2 L $ and thus $\|\sum_{j=1}^m \mathbb{E}_{\theta}[ Y_j^2  ] \| \le 2 \lambda_n$. By the matrix Bernstein inequality \cite{tropp2010matrixmtg}, with probability at least $1-n^{-1}$,
\begin{align*}
 \|\tilde{L}- \mathbb{E}_{\theta}[\tilde{L} ] \| \le 2 \sqrt{ \lambda_n \log n} + \frac{2}{3} \log n.
\end{align*}
By the assumption that $\lambda_n \ge C \log n$ for some sufficiently large constant $C$, $\|\tilde{L}- \mathbb{E}_\theta [\tilde{L}]  \|  \le 4 \sqrt{\lambda_n \log n}$. It follows from  \prettyref{eq:HessianLowerbound} and \prettyref{eq:boundLaplacian} that
\begin{align*}
 \lambda_2(-H(\theta)) \ge \frac{1}{e^{2b}} \lambda_2 (\tilde{L} ) & \ge \frac{1}{e^{2b}} \left( \frac{1}{4e^{2b} }  \lambda_2 - 4  \sqrt{\lambda_n \log n} \right).
\end{align*}

\end{proof}

\begin{proof}[Proof of \prettyref{thm:MLEUpperBound}]
Define $\Delta=\widehat{\theta}_{\ML}-\theta^\ast$. It follows from the definition that $\Delta$ is orthogonal to the all-one vector. By the definition of the ML estimator, $\calL(\widehat{\theta}_{\ML}) \ge \calL(\theta^\ast)$ and thus
\begin{align}
\calL(\hat{\theta}_{\ML})- \calL(\theta^\ast) - \langle \nabla \calL(\theta^\ast), \Delta \rangle \ge - \langle \nabla \calL(\theta^\ast), \Delta \rangle \ge - \|\nabla \calL(\theta^\ast)\|_2  \|\Delta\|_2, \label{eq:loglikelihoodgradient}
\end{align}
where the last inequality holds due to the Cauchy-Schwartz inequality. By the Taylor expansion, there exists a $\theta= a \widehat{\theta}_{\ML}+ (1-a) \theta^\ast $ for some $a \in [0,1]$ such that
\begin{align}
\calL(\hat{\theta}_{\ML})- \calL(\theta^\ast) - \langle \nabla \calL(\theta^\ast), \Delta \rangle = \frac{1}{2}    \Delta^\top H (\theta) \Delta
\le - \frac{1}{2} \lambda_2(-H(\theta)) \|\Delta\|_2^2, \label{eq:LoglikelihoodHessian}
\end{align}
where the last inequality holds because the Hessian matrix $-H(\theta)$ is positive semi-definite with $H(\theta) \mathbf{1} =\mathbf{0}$ and $\Delta^\top \mathbf{1}=0$.
Combining \prettyref{eq:loglikelihoodgradient} and \prettyref{eq:LoglikelihoodHessian},
\begin{align}
\|\Delta\|_2 \le  2 \|\nabla \calL(\theta^\ast)\|_2 /\lambda_2(- H(\theta)). \label{eq:MLEUpperBound}
\end{align}
Note that $\theta \in \Theta_b$ by definition. The theorem follows by \prettyref{lmm:Gradient} and \prettyref{lmm:Hessian}.

\subsection{Proof of \prettyref{cor:MLEUpperBound}}
Recall that $L= \sum_{j=1}^m L_j$.
Observe that
$
\mathbb{E}[L_j]= \frac{k_j}{n-1} \left( I- \frac{1}{n} \mathbf{1}\mathbf{1}^\top \right).
$
Define $Z_j=L_j- \mathbb{E}[L_j]$. Then $Z_1, \ldots, Z_m$ are independent symmetric random matrices with zero mean. Note that
\begin{align*}
\|Z_j\| \le  \|L_j\| + \|\mathbb{E}[L_j]\| \le  \frac{k_j}{k_j-1} + \frac{k_j}{n-1} \le 4.
 \end{align*}
Moreover,
\begin{align*}
\mathbb{E}[Z_j^2]= \frac{k_j^2}{(k_j-1)(n-1)} \left( I- \frac{1}{n} \mathbf{1}\mathbf{1}^\top \right)- \frac{k_j^2}{(n-1)^2} \left( I- \frac{1}{n} \mathbf{1}\mathbf{1}^\top \right).
\end{align*}
Therefore, $\| \sum_{j=1}^m \mathbb{E}[Z_j^2] \| \le  \frac{2mk}{n-1}$. By the matrix Bernstein inequality \cite{tropp2010matrixmtg}, with probability at least $1-n^{-1}$,
\begin{align*}
 \| L -  \mathbb{E}[L] \| \le 2\sqrt{\frac{ mk  \log n }{n-1}} + \frac{8}{3} \log n  \le 4 \sqrt{\frac{m k\log n }{n-1} } \le \frac{mk}{2(n-1)}.
\end{align*}
where the last two inequalities follow from the assumption that $mk \ge C \log n$ for some sufficiently large constant $C$.
Since $\mathbb{E}[L]=\frac{mk}{n-1} \left( I- \frac{1}{n} \mathbf{1}\mathbf{1}^\top \right)$, the smallest eigenvalue of $\mathbb{E}[L]$ is zero and all the other eigenvalues equal $\frac{mk}{n-1}$. It follows that
\begin{align*}
|\lambda_i - \frac{mk}{n-1}| \le  \| L -  \mathbb{E}[L] \| \le \frac{mk}{2(n-1)}, \; 2 \le i \le n,
\end{align*}
and thus $ \lambda_2 \ge \frac{mk}{2(n-1)}$ and $\lambda_n \le \frac{3mk}{2(n-1)}$.  By the assumption that $mk \ge Ce^{2b} \log n$ for some sufficiently large constant $C$,
$
\lambda_2 - 16 e^{2b} \sqrt{\lambda_n \log n} \ge \frac{mk}{4n}.
$
Then the corollary follow from  \prettyref{thm:MLEUpperBound}.
\end{proof}

\subsection{Proof of \prettyref{cor:IBUpperBound}}
Without loss of generality, assume $k_j$ is even for all $j \in [m]$. After the random \IB, there are $mk/2$ independent
pairwise comparisons and let $L$ denote the Laplacian of the comparison graph after the breaking. 
Recall that $L= \sum_{j=1}^m L_j$. With random \IB, we have
$
\mathbb{E}[L_j]= \frac{k_j}{n-1} \left( I- \frac{1}{n} \mathbf{1}\mathbf{1}^\top \right).
$
Define $Z_j=L_j- \mathbb{E}[L_j]$. Then $Z_1, \ldots, Z_m$ are independent symmetric random matrices with zero mean. Moreover,
\begin{align*}
\|Z_j\| \le  \|L_j\| + \|\mathbb{E}[L_j]\| \le  2 + \frac{k_j}{n-1} \le 4,
 \end{align*}
and
\begin{align*}
\mathbb{E}[Z_j^2]= \frac{2k_j}{n-1} \left( I- \frac{1}{n} \mathbf{1}\mathbf{1}^\top \right)- \frac{k_j^2}{(n-1)^2} \left( I- \frac{1}{n} \mathbf{1}\mathbf{1}^\top \right).
\end{align*}
Therefore, $\| \sum_{j=1}^m \mathbb{E}[Z_j^2] \| \le  \frac{2mk}{n-1}$.
Following the same argument for proving \prettyref{cor:MLEUpperBound}, we can show that $ \lambda_2(L_{\IB}) \ge \frac{mk}{2(n-1)}$ and the corollary follows by \prettyref{thm:MLEUpperBound} with $k=2$.

\subsection{Proof of \prettyref{thm:FBUpperBound}}
It follows from the definition of $\calL(\theta)$ given by \prettyref{eq:loglikelihoodbreaking} that
\begin{equation}  \label{eq.gradL_breaking}
\nabla_i \calL(\theta^*) = \sum_{j: i \in S_j} \frac{1}{k_j-1}   \sum_{i' \in S_j: i'\neq i} \left[ \1{\sigma_j^{-1}(i)< \sigma_j^{-1}(i') } - \frac{\exp(\theta^\ast_i)}{\exp(\theta^\ast_i)+\exp(\theta^\ast_{i'} ) } \right] := \sum_{j: i \in S_j} Y_j,
\end{equation}
which is a sum of $d_i$ independent random variables with mean zero and bounded by $1$. By Hoeffding's inequality, $|\nabla_i L(\theta^*)| \le \sqrt{d_i \log n} $ with probability at least $1-2n^{-2}$. By union bound, $\|\nabla L(\theta^\ast)\|_2 \le \sqrt{m k \log n} $ with probability at least $1-2n^{-1}$.
The Hessian matrix is given by
\begin{align*}
H(\theta)=- \sum_{j=1}^m  \frac{1}{2(k_j-1)}\sum_{i,i' \in S_j}  (e_i-e_{i'} ) (e_i - e_{i'} )^\top  \frac{\exp(\theta_i+ \theta_{i'} ) }{ \left[ \exp( \theta_{i}) +  \exp(\theta_{i'} ) \right]^2 }.
\end{align*}
If $|\theta_i | \le b, \forall i \in [n]$,
$
\frac{\exp(\theta_i+ \theta_{i'} ) }{ \left[ \exp( \theta_{i}) +  \exp(\theta_{i'} ) \right]^2 } \ge \frac{e^{2b}}{(1+e^{2b})^2}.
$
It follows that $-H(\theta) \ge \frac{e^{2b}}{(1+e^{2b})^2} L$ for $\theta \in \Theta_b$ and the theorem follows from \prettyref{eq:MLEUpperBound}.

\bibliographystyle{IEEEtran}
\bibliography{ranking}
\end{document}